\documentclass{article}

\usepackage{microtype}
\usepackage{graphicx}
\usepackage{subcaption}
\usepackage{booktabs} 

\usepackage{hyperref}

\usepackage[preprint]{icml2026}

\usepackage{amsmath}
\usepackage{amssymb}
\usepackage{mathtools}
\usepackage{amsthm}

\usepackage{algorithm}
\usepackage{algorithmic}

\usepackage{xcolor}
\usepackage[most]{tcolorbox}

\definecolor{lightskyblue}{RGB}{230, 245, 255}

\usepackage[capitalize,noabbrev]{cleveref}

\theoremstyle{plain}
\newtheorem{theorem}{Theorem}[section]

\newtheorem{lemma}[theorem]{Lemma}

\theoremstyle{definition}

\theoremstyle{remark}

\usepackage[textsize=tiny]{todonotes}

\icmltitlerunning{Back to Basics: Revisiting Exploration in Reinforcement Learning for LLM Reasoning via Generative Probabilities}

\begin{document}

\twocolumn[
  \icmltitle{Back to Basics: Revisiting Exploration in Reinforcement Learning for LLM Reasoning via Generative Probabilities}



  \icmlsetsymbol{equal}{*}

  \begin{icmlauthorlist}
    \icmlauthor{Pengyi Li}{yyy}
    \icmlauthor{Elizaveta Goncharova}{yyy}
    \icmlauthor{Andrey Kuznetsov}{yyy}
    \icmlauthor{Ivan Oseledets}{sch}
  \end{icmlauthorlist}

  \icmlaffiliation{yyy}{FusionBrain Lab, Russia}
  \icmlaffiliation{sch}{Institute of Numerical Mathematics, Russia}

  \icmlcorrespondingauthor{Pengyi Li}{li.pengyi@fusionbrainlab.com}


  \vskip 0.3in
]



\printAffiliationsAndNotice{}  

\begin{abstract}
Reinforcement Learning with Verifiable Rewards (RLVR) has emerged as an indispensable paradigm for enhancing reasoning in Large Language Models (LLMs). However, standard policy optimization methods, such as Group Relative Policy Optimization (GRPO), often converge to low-entropy policies, leading to severe mode collapse and limited output diversity. We analyze this issue from the perspective of sampling probability dynamics, identifying that the standard objective disproportionately reinforces the highest-likelihood paths, thereby suppressing valid alternative reasoning chains.
To address this, we propose a novel Advantage Re-weighting Mechanism (ARM) designed to equilibrate the confidence levels across all correct responses. By incorporating Prompt Perplexity and Answer Confidence into the advantage estimation, our method dynamically reshapes the reward signal to attenuate the gradient updates of over-confident reasoning paths, while redistributing probability mass toward under-explored correct solutions. Empirical results demonstrate that our approach significantly enhances generative diversity and response entropy while maintaining competitive accuracy, effectively achieving a superior trade-off between exploration and exploitation in reasoning tasks. Empirical results on Qwen2.5 and DeepSeek models across mathematical and coding benchmarks show that ProGRPO significantly mitigates entropy collapse. Specifically, on Qwen2.5-7B, our method outperforms GRPO by 5.7\% in Pass@1 and, notably, by 13.9\% in Pass@32, highlighting its superior capability in generating diverse correct reasoning paths.
\end{abstract}

\section{Introduction}

In recent years, Large Language Models (LLMs) have made significant progress through Reinforcement Learning (RL)-driven post-training \citep{reinforcement, ppo, dpo, rlhf}, particularly in complex reasoning tasks. As a simple yet highly efficient training paradigm, Reinforcement Learning with Verifiable Rewards (RLVR) \citep{rlvrtulu, deepseekr1} leverages explicit verification signals to stably induce the generation of longer Chain-of-Thought (CoT) reasoning paths \citep{grpo, openaio1, team2025kimi}. This, in turn, leads to substantial performance gains on high-difficulty reasoning tasks, consistent with prior findings on the effectiveness of CoT reasoning and test-time scaling \citep{CoT, ttscaling}.

However, although RLVR can effectively improve task success rates, its training process is often accompanied by obvious entropy collapse and mode collapse phenomena, leading to reasoning paths generated by the model being highly concentrated on a few dominant solutions. This issue inherently stems from the reward-weighted likelihood maximization objective: this objective continuously amplifies the probability mass of high-reward trajectories during optimization, thereby compressing the probability space occupied by low-frequency but equally valid reasoning paths, weakening the model's exploration capabilities \citep{doseRLpassk}.

Addressing the aforementioned issues, existing works have attempted to mitigate mode collapse by introducing entropy regularization \citep{entropyR2, entropyR1, cui2025entropy, wang2025reinforcement}, clip-higher \citep{yu2025dapo}, dynamic clipping strategies \citep{yang2025dcpo}, or high-entropy token promotion mechanisms \citep{2080tokens}. However, these methods remain largely limited to local modifications within the reward maximization framework, making it difficult to fundamentally improve the ability to model diverse reasoning paths. On the other hand, recent research indicates that applying penalties only to incorrect trajectories while maintaining a relatively flat reward structure for correct ones can, to a certain extent, increase the diversity of the solution space \citep{negativeRL}. This phenomenon further demonstrates that the relative probability structure among reasoning paths plays a critical role in shaping the model's exploration behavior.

From the perspective of generative modeling, the reasoning process of LLMs is essentially a token-by-token probabilistic sampling process. When the policy becomes overly deterministic on a few trajectories, the diversity of the sampling distribution inevitably drops \citep{li2025confidence}. Based on this observation, we propose a new reinforcement learning paradigm, Probabilistic based GRPO (ProGRPO), which re-examines the construction of Advantage from the perspective of probability distributions.

Specifically, ProGRPO utilizes the internal probability signals of the prompt and the generated answer from the LLM, combined with verifiable rewards, to reshape the Advantage distribution, thereby implicitly reshaping the effective reward-weighted trajectory distribution optimized by the policy gradient. This re-weighting mechanism based on probability structure can effectively alleviate the entropy collapse problem and significantly enhance the diversity of reasoning paths and training stability.

Our main contributions include:

\begin{itemize}
    \item \textbf{Methodology:} We propose ProGRPO, a principled extension of GRPO that incorporates a novel \emph{Advantage Re-weighting Mechanism} (\textbf{ARM}). By introducing confidence-aware signals into the advantage function, our method achieves targeted exploration without compromising training stability.
    \item \textbf{Broad Effectiveness:} We validate \textbf{ProGRPO} across diverse reasoning and code generation benchmarks using Qwen2.5 (7B, 32B) and DeepSeek models. Our method consistently outperforms mainstream baselines like GRPO and FlowRL, demonstrating strong scalability and generalization across different model sizes.

    \item \textbf{OOD Robustness:} Beyond standard benchmarks, \textbf{ProGRPO} exhibits superior Out-of-Distribution (\textbf{OOD}) adaptability, maintaining robust performance on unseen data distributions.

    \item \textbf{Significant Gains:} \textbf{ProGRPO} substantially enhances both accuracy and output diversity. On Qwen2.5-7B, it improves Pass@1 and Pass@32 by \textbf{5.7\%} and \textbf{13.9\%} respectively over GRPO (and 8.0\% / 7.5\% over FlowRL), highlighting its superior exploration efficiency.
\end{itemize}
\section{Preliminaries}

\subsection{REINFORCE}
REINFORCE \citep{reinforcement} is the classic policy gradient algorithm. Its objective is to maximize the expected cumulative reward. However, the gradient estimation suffers from high variance, leading to training instability.
\begin{equation}
\mathcal{J}_{\text{REINFORCE}}(\theta) = \mathbb{E}_{\tau \sim \pi_\theta} \big[ R(\tau) \, \log \pi_\theta(\tau) \big]
\end{equation}
\subsection{Proximal Policy Optimization (PPO)}
PPO \citep{ppo} stabilizes training by enforcing a trust region constraint via a clipping mechanism, which restricts the size of the policy update.
\begin{equation}
\begin{split}
& \mathcal{J}_{\text{PPO}}(\theta) = \mathbb{E}_{q \sim D, o \sim \pi_{\theta_{\text{old}}}} \\
& \Bigg[ \frac{1}{|o|} \sum_{t=1}^{|o|} \min \bigg( r_t(\theta) A_t, \\
& \operatorname{clip}\big(r_t(\theta), 1-\epsilon, 1+\epsilon\big) A_t \bigg) \Bigg]
\end{split}
\end{equation}
PPO requires an additional value function (Critic) to estimate the advantage $A_t$, which incurs significant computational overhead.

\subsection{Group Relative Policy Optimization (GRPO)}
GRPO \citep{grpo} eliminates the need for a value function by using verifiable rewards and group-based relative advantages. It samples a group of outputs $\{o_i\}_{i=1}^G$ for each query $q$ and uses the group mean as the baseline.
\begin{equation}
\begin{split}
& \mathcal{J}_{\text{GRPO}}(\theta) = \mathbb{E}_{q \sim D, \{o_i\}_{i=1}^G \sim \pi_{\theta_{\text{old}}}} \\
& \Bigg[ \frac{1}{G} \sum_{i=1}^G \frac{1}{|o_i|} \sum_{t=1}^{|o_i|} \min \bigg( r_{i,t}(\theta) A_i, \\
& \operatorname{clip}\big(r_{i,t}(\theta), 1-\epsilon, 1+\epsilon\big) A_i \bigg) \Bigg]
\end{split}
\end{equation}
where the advantage $A_i$ is computed by normalizing the rewards within the group: $A_i = \frac{R_i - \text{mean}(R)}{\text{std}(R)}$.
\section{Methodology}

\subsection{Advantage Re-weighting Mechanism (AMR)}

We redefine the advantage by incorporating sample-level signals derived from the model itself, using the low-probability token length normalized likelihood as a confidence score.
\begin{equation} \label{main_equa}
    \tilde{A}_{i} =
    \begin{cases} 
        A_{i}, \\
        \quad \text{if } \sum_{k=1}^{G} r_{i,k} = 0 \\[1em]
        A_{i} + \alpha \left(c_\theta(q_i) - c_\theta(o_{i} \mid q_i)\right), \\ 
        \quad \text{otherwise}
    \end{cases}
\end{equation}
Here, $c(q_i)$ denotes the model's \textbf{confidence} on the current prompt. Rather than relying on heuristic assumptions, we ground this design in the principles of Curriculum Reinforcement Learning \citep{H2E}, which posits that the difficulty of training samples (ranging from simple to hard) significantly impacts model optimization. Consequently, we incorporate $c(q_i)$ as a dynamic control term to regulate the training process based on the model's familiarity with the prompt. 
\begin{equation} \label{prompt_conf}
    c_\theta(q_i)
    =
    \exp\!\left(
    \frac{1}{|\mathcal{T}_i^{\text{low}}|}
    \sum_{t \in \mathcal{T}_i^{\text{low}}}
    \log p_\theta\!\left(q_{i,t} \mid q_{i,<t}\right)
    \right)
\end{equation}
$c_\theta(o_j \mid q_i)$ represents the model's \textbf{confidence} in generating the answer $o_j$ for prompt $q_i$.
\begin{equation} \label{answer_conf}
    c_\theta(o_j \mid q_i)
    =
    \exp\Bigg(
    \frac{1}{|\mathcal{T}_i^{\text{low}}|}
    \sum_{t \in \mathcal{T}_i^{\text{low}}} 
    \log p_\theta\big(o_{j,t} \mid q_i,\, o_{j,<t}\big)
    \Bigg)
\end{equation}
This Advantage reweighting indirectly reshapes the effective reward distribution, allowing us to score reasoning trajectories rather than simply maximize rewards. The reason for not modifying the reward directly is that within a group, all answers might be correct or incorrect; directly penalizing or rewarding them could distort the update signal, compromising the stability and effectiveness of model training. By adjusting the Advantage instead, we preserve meaningful gradient signals while encouraging diverse and accurate reasoning paths.

\subsection{Low-Probability Token Length Normalization}

We observe that applying length normalization to the full sequence likelihood can be suboptimal for reward modeling. Following the insights from \cite{2080tokens}, predictive uncertainty is typically concentrated in a small fraction of generation steps. Specifically, roughly 20\% of token positions substantially influence the subsequent reasoning path, while at the remaining positions, the model's next-token distribution is sharply peaked, with the top candidate often receiving a probability above 0.9. 

In our framework, applying length normalization over the entire sequence would disproportionately dilute the reward signal by including these "trivial" high-confidence tokens, leading to a weak and less informative training signal. To mitigate this, we define a critical subset of tokens, denoted as $\mathcal{T}^{\text{low}}_{o_i}$, which comprises the approximately 20\% of positions in the response $o_i$ that exhibit the highest predictive uncertainty. Consequently, we apply this selective length normalization to the confidence scores formulated in Equations~\ref{prompt_conf} and \ref{answer_conf}. 

By focusing on this informative subset, we preserve meaningful confidence variations that more directly reflect the model's reasoning quality, thereby providing a more robust signal for policy optimization.

\subsection{ProGRPO}

Finally, our overall objective function is given by Equation~\ref{Eq_ProGRPO}.
\begin{equation} \label{Eq_ProGRPO}
\begin{split}
& \mathcal{J}_{\text{ProGRPO}}(\theta) = \mathbb{E}_{(q,a)\sim\mathcal{D}, \{o_i\}_{i=1}^G \sim \pi_{\theta_{\text{old}}}(\cdot|q)} \\
& \Bigg[ \frac{1}{\sum_{i=1}^G |o_i|} \sum_{i=1}^G \sum_{t=1}^{|o_i|} \min \bigg( r_{i,t}(\theta) \tilde{A}_{i}, \\
& \operatorname{clip}\big(r_{i,t}(\theta), 1-\varepsilon_{\text{low}}, 1+\varepsilon_{\text{high}}\big) \tilde{A}_{i} \bigg) \Bigg]
\end{split}
\end{equation}

The final pseudocode is presented in Algorithm~\ref{alg:progrpo}. A detailed theoretical justification of ProGRPO is provided in
Appendix~\ref{Theoretical_Justification}.

\begin{algorithm}[h!]
\caption{ProGRPO: Probabilistic Group Relative Policy Optimization}
\label{alg:progrpo}
\begin{algorithmic}[1]
    \STATE {\bfseries Input:} Dataset $\mathcal{D}$, Policy Model $\pi_\theta$, Reference Model $\pi_{\text{ref}}$
    \STATE {\bfseries Hyperparams:} Group size $G$, Learning rate $\eta$, Weight $\alpha$, Clip $\varepsilon_{\text{low}}, \varepsilon_{\text{high}}$
    
    \WHILE{not converged}
        \STATE Sample batch of prompts $Q \sim \mathcal{D}$
        \STATE Initialize batch loss $L_{\text{batch}} = 0$
        
        \FOR{each prompt $q$ in $Q$}
            \STATE \textbf{1. Sampling Phase}
            \STATE Generate $G$ outputs $\{o_1, \dots, o_G\}$ from $\pi_{\theta_{\text{old}}}(\cdot \mid q)$
            \STATE Compute rewards $R = \{r_1, \dots, r_G\}$
            
            \STATE \textbf{2. Standard GRPO Advantage}
            \STATE Compute $\mu = \text{mean}(R)$ and $\sigma = \text{std}(R)$
            \STATE $A_i = \frac{r_i - \mu}{\sigma + \delta}$ for $i \in \{1 \dots G\}$
            
            \STATE \textbf{3. Advantage Re-weighting (AMR)}
            \IF{$\sum_{k=1}^G r_k = 0$ \textbf{or} $\sum_{k=1}^G r_k = G$}
                \STATE $\tilde{A}_i \leftarrow A_i$ for all $i \in \{1 \dots G\}$
            \ELSE
                \STATE \textit{// Calculate Prompt Confidence (Eq. ~\ref{prompt_conf})}
                \STATE Identify low-prob tokens $\mathcal{T}^{\text{low}}_q$ in prompt $q$
                \STATE Compute $c_\theta(q)$

                \FOR{$i = 1$ \textbf{to} $G$}
                    \STATE \textit{// Calculate Answer Confidence (Eq. ~\ref{answer_conf})}
                    \STATE Identify low-prob tokens $\mathcal{T}^{\text{low}}_{o_i}$ in answer $o_i$
                    \STATE Compute $c_\theta(o_i \mid q)$
                    \STATE \textit{// Apply Re-weighting (Eq. ~\ref{main_equa})}
                    \STATE $\tilde{A}_i \leftarrow A_i + \alpha \cdot (c_\theta(q) - c_\theta(o_i \mid q))$
                \ENDFOR
            \ENDIF
            
            \STATE \textbf{4. Loss Computation (Eq. ~\ref{Eq_ProGRPO})}
            \STATE Initialize prompt loss $L_q = 0$
            \FOR{$i = 1$ \textbf{to} $G$}
                \FOR{$t = 1$ \textbf{to} $|o_i|$}
                    \STATE Ratio $r_{i,t}(\theta) = \frac{\pi_\theta(o_{i,t} \mid q, o_{i,<t})}{\pi_{\theta_{\text{old}}}(o_{i,t} \mid q, o_{i,<t})}$
                    \STATE $L_{\text{surr}} = \min\Big( r_{i,t}(\theta) \tilde{A}_i, $
                    \STATE \quad $\operatorname{clip}\big(r_{i,t}(\theta), 1-\varepsilon_{\text{low}}, 1+\varepsilon_{\text{high}}\big) \tilde{A}_i \Big)$
                    \STATE $L_q \leftarrow L_q + L_{\text{surr}}$
                \ENDFOR
            \ENDFOR
            \STATE $L_q \leftarrow \frac{1}{\sum_{i=1}^G |o_i|} L_q$ 
            \STATE $L_{\text{batch}} \leftarrow L_{\text{batch}} + L_q$
        \ENDFOR
        \STATE Update parameters $\theta$ by minimizing $-L_{\text{batch}}$
    \ENDWHILE
\end{algorithmic}
\end{algorithm}

\section{Experiments}

\subsection{Experimental Settings}

\begin{table}[ht!] 
\centering
\caption{Training Hyperparameters} \label{Traing_hyperparameters}
\begin{tabular}{l l}
\toprule
\textbf{Hyperparameter} & \textbf{Value} \\
\midrule
Advantage Estimator & GRPO \\
Use KL Loss & No \\
Use Entropy Regularization & No \\
Train Batch Size & 512 \\
Max Response Length & 8092 \\
PPO Mini-batch Size & 32 \\
Clip Ratio Range & [0.8, 1.28] \\
Learning Rate & $1 \times 10^{-6}$ \\
Sampling Temperature & 1.0 \\
Number of Rollouts ($N$) & 8 \\
Reward Function & DAPO \cite{yu2025dapo} \\
\bottomrule
\end{tabular}
\end{table}

\textbf{Training Setup.}
Our experiments are conducted under the GRPO framework; detailed hyperparameter settings are provided in Table~\ref{Traing_hyperparameters}.

We do not employ any specially designed or task-specific prompts in this work. All models are trained and evaluated using the default prompting setup of the underlying language model.

\textbf{Training Dataset.}
We conduct experiments in two domains: mathematics and code generation. For mathematics, we train on the \textbf{DAPO} dataset \citep{yu2025dapo}. For code generation, we use the training split of the \textbf{DeepCoder} dataset \citep{luo2025deepcoder}.

\textbf{Base Models.}
Our experiments involve multiple model scales and families. We utilize \textbf{Qwen2.5-7B, Qwen2.5-32B} \citep{qwen2.5} and \textbf{DeepSeek-R1-Distill-Qwen-1.5B} \cite{deepseekr1} for general reasoning tasks, while leveraging \textbf{DeepSeek-R1-Distill-Qwen-7B} \cite{deepseekr1} specifically for the code domain.

\subsection{Evaluation}

\textbf{Math Domain.} We evaluate our method on a set of widely used mathematical reasoning benchmarks, including AIME2024 \citep{aime24&amc23}, AIME2025 \citep{aime25}, AMC23 \citep{aime24&amc23}, MATH500 \citep{math500}, Minerva \citep{Minerva_Math}, and OlympiadBench \citep{olympiadbench}. 

\textbf{Code Domain.} For code-related tasks, we conduct experiments on LiveCodeBench \citep{livecodebench}, CodeForces \citep{penedo2025codeforces}, and HumanEval+ \citep{humaneval}. 

For \textbf{out-of-distribution (OOD) evaluation} in the general domain, we use \textbf{GPQA} \citep{rein2024gpqa} and \textbf{MMLU-Pro} \citep{wang2024mmlu}.

During the evaluation, we set the sampling temperature to 0.6 and top-p to 0.95. We report performance using Pass@1 and Pass@k as the primary evaluation metrics.

For the Qwen2.5 series \citep{bai2023qwen}, we perform evaluations on reasoning benchmarks with a maximum output length of 8K tokens. For the DeepSeek-R1-Distill-Qwen series \citep{deepseekr1}, we use a maximum output length of 8K tokens for code-domain tasks and 32K tokens for reasoning tasks, reflecting their extended context capabilities.

\subsection{Results}

\begin{table*}[htp!]
    \centering
    \caption{Main results across six mathematical reasoning benchmarks. Each cell reports Pass@1 / Pass@32 (\%). GRPO w/ KL-Cov \citep{cui2025entropy} exhibits high sensitivity to optimization hyperparameters. Although all experiments are conducted using the original implementation without modifications, training instability is occasionally observed, which negatively affects the final results.}
    \label{main_math_results}
    \resizebox{\textwidth}{!}{%
    \begin{tabular}{llccccccc}
        \toprule
        \textbf{Method} & \textbf{AIME 2024} & \textbf{AIME 2025} & \textbf{AMC 23} & \textbf{MATH 500} & \textbf{Minerva} & \textbf{OlympiadBench} & \textbf{Average} \\
        \midrule
        \multicolumn{8}{c}{\textbf{Qwen2.5-7B (Max Response Len = 8K tokens)}} \\
        \midrule
        Baseline 
        & 5.4 / 30.0 
        & 2.5 / 20.0 
        & 32.4 / 82.5 
        & 54.7 / 92.6 
        & 22.0 / 54.4
        & 24.6 / 63.1 
        & 23.6 / 57.1 \\
        GRPO
        & 9.2 / 26.7
        & 6.1 / 30.0
        & 65.5 / 80.0
        & 75.3 / 87.0
        & \textbf{33.8} / 51.1
        & 35.6 / 53.6
        & 37.6 / 54.7 \\
        GRPO w KL-Cov
        & 6.1 / 26.7
        & 8.4 / 36.7
        & 46.3 / 80.0
        & 56.9 / 87.6
        & 23.2 / \textbf{54.8}
        & 25.7 / 60.2
        & 27.8 / 57.7 \\
        FlowRL
        & 14.6 / 33.3 
        & 10.4 / 40.0 
        & 54.0 / 85.0
        & 66.9 / 87.0 
        & 30.9 / 54.4
        & 35.0 / 60.7 
        & 35.3 / 61.0 \\
        \textbf{ProGRPO (Ours)} 
        & \textbf{21.3} / \textbf{53.3}
        & \textbf{15.9} / \textbf{50.0}
        & \textbf{67.2} / \textbf{92.5}
        & \textbf{80.5} / \textbf{94.2}
        & 32.0 / 53.3
        & \textbf{42.7} / \textbf{67.5}
        & \textbf{43.3} / \textbf{68.5} \\
        
        \midrule
        \multicolumn{8}{c}{\textbf{Qwen2.5-32B (Max Response Len = 8K tokens)}} \\
        \midrule
        
        Baseline
        & 5.7 / 43.3
        & 2.2 / 26.7
        & 29.7 / 87.5
        & 52.0 / 91.8
        & 27.3 / 59.6
        & 22.4 / 68.4
        & 23.2 / 62.9 \\

        GRPO
        & 16.8 / 43.3
        & 13.9 / 26.7
        & 80.1 / 95.0
        & \textbf{84.7} / 92.4
        & \textbf{42.3} / 55.9
        & 49.7 / 65.1
        & 47.9 / 63.1 \\
        
        \textbf{ProGRPO (Ours)}
        & \textbf{34.6} / \textbf{60.0}
        & \textbf{24.8} / \textbf{46.7}
        & \textbf{80.6} / \textbf{97.5}
        & 82.8 / \textbf{96.2}
        & 39.3 / \textbf{58.5}
        & \textbf{54.1} / \textbf{74.5}
        & \textbf{52.7} / \textbf{72.2} \\
        
        \midrule
        \multicolumn{8}{c}{\textbf{DeepSeek-R1-Distill-Qwen-1.5B (Max Response Len = 32K tokens)}} \\
        \midrule
        
        Baseline
        & 31.6 / 76.7
        & 24.9 / 60.0
        & 71.6 / 95.0
        & 74.3 / 95.6
        & 26.1 / 56.3
        & 49.1 / 80.1
        & 46.3 / 77.3 \\

        GRPO
        & 26.6 / 70.0
        & 28.0 / 56.7
        & 77.9 / \textbf{97.5}
        & 80.6 / 96.6
        & 29.9 / 56.3
        & 53.1 / 79.2
        & 49.4 / 76.1 \\

        \textbf{ProGRPO (Ours)} 
        & \textbf{46.0} / \textbf{76.7}
        & \textbf{33.6} / \textbf{60.0}
        & \textbf{86.3} / 95.0
        & \textbf{87.0} / \textbf{98.2}
        & \textbf{34.8} / \textbf{58.1}
        & \textbf{62.0} / \textbf{81.8}
        & \textbf{58.3} / \textbf{78.3} \\
        
        \bottomrule
    \end{tabular}%
    }
\end{table*}

\begin{table}[!htb]
\centering
\caption{Performance comparison on code reasoning benchmarks. All models are evaluated with a maximum response length of 8K tokens. * We state that FlowRL's results were reproduced using weights released by Hugging Face, however, this does not affect the overall conclusions of our study.}
\label{main_code_results}
\resizebox{\columnwidth}{!}{%
\begin{tabular}{lccccc}
\toprule
\textbf{Models} 
& \multicolumn{2}{c}{\textbf{LiveCodeBench}} 
& \multicolumn{2}{c}{\textbf{CodeForces}} 
& \textbf{HumanEval+} \\
& Avg@16 & Pass@16 & Rating & Pct. & Avg@16 \\
\midrule
\multicolumn{6}{c}{\textbf{DeepSeek-R1-Distill-Qwen-7B (Max Response Len = 8K tokens)}} \\
\midrule
Backbone 
& 30.96 
& 48.03 
& 764.93 
& 8.2 
& 80.33 \\

GRPO
& 34.94
& 53.76
& 1243.24
& 60.7
& 83.32\\

FlowRL
& 33.98
& 51.97
& 1129.58
& 48.5 
& 82.67 \\

\midrule
\textbf{ProGRPO (Ours)} 
& \textbf{36.47}
& \textbf{54.12}
& \textbf{1422.49}
& \textbf{75.4}
& \textbf{84.01} \\
\bottomrule
\end{tabular}%
}
\end{table}




Our primary experimental results are summarized in Tables~\ref{main_math_results} and~\ref{main_code_results}. Across both mathematical reasoning and code generation domains, \textbf{ProGRPO} consistently outperforms the direct reward maximization baseline (GRPO) as well as the reward matching approach proposed by FlowRL.

Mathematical reasoning.
As shown in Table~\ref{main_math_results}, \textbf{ProGRPO} achieves substantial improvements across all evaluated benchmarks and model scales. For Qwen2.5-7B, \textbf{ProGRPO} attains an average Pass@1 of 43.3\%, improving over GRPO by +5.7\% and over FlowRL by +8.0\%. The gains are even more pronounced in the multi-sample regime, where \textbf{ProGRPO} reaches an average Pass@32 of 68.5\%, surpassing GRPO and FlowRL by +13.8 and +7.5\%, respectively. Notably, large margins are observed on challenging benchmarks such as AIME 2024 (+12.1 Pass@1 over FlowRL) and OlympiadBench (+7.7 Pass@1 over FlowRL).

For Qwen2.5-32B, \textbf{ProGRPO} further scales favorably, achieving an average Pass@1 of 52.7\%, which is +4.8\% higher than GRPO. Similar trends are observed for DeepSeek-R1-Distill-Qwen-1.5B, where \textbf{ProGRPO} improves the average Pass@1 from 49.4\% to 58.3\%, demonstrating that our method remains effective even for smaller distilled models. Overall, these results indicate that \textbf{ProGRPO} delivers robust and consistent gains across model sizes and mathematical reasoning tasks.

Code generation.
Table~\ref{main_code_results} presents the results on code reasoning benchmarks. On LiveCodeBench, \textbf{ProGRPO} achieves an Avg@16 score of 36.47 and Pass@16 of 54.12, outperforming GRPO by +1.53 and +0.36, respectively. On CodeForces, \textbf{ProGRPO} yields a substantial improvement, reaching a rating of 1422.49, which exceeds GRPO by nearly +180 rating and FlowRL by +293, corresponding to a percentile increase to 75.4\%. Additionally, \textbf{ProGRPO} attains the best performance on HumanEval+, achieving an Avg@16 score of 84.01\%, further confirming its effectiveness in complex logic and syntax generation tasks.

\begin{figure*}[!htb]
    \centering
    \includegraphics[width=1\linewidth]{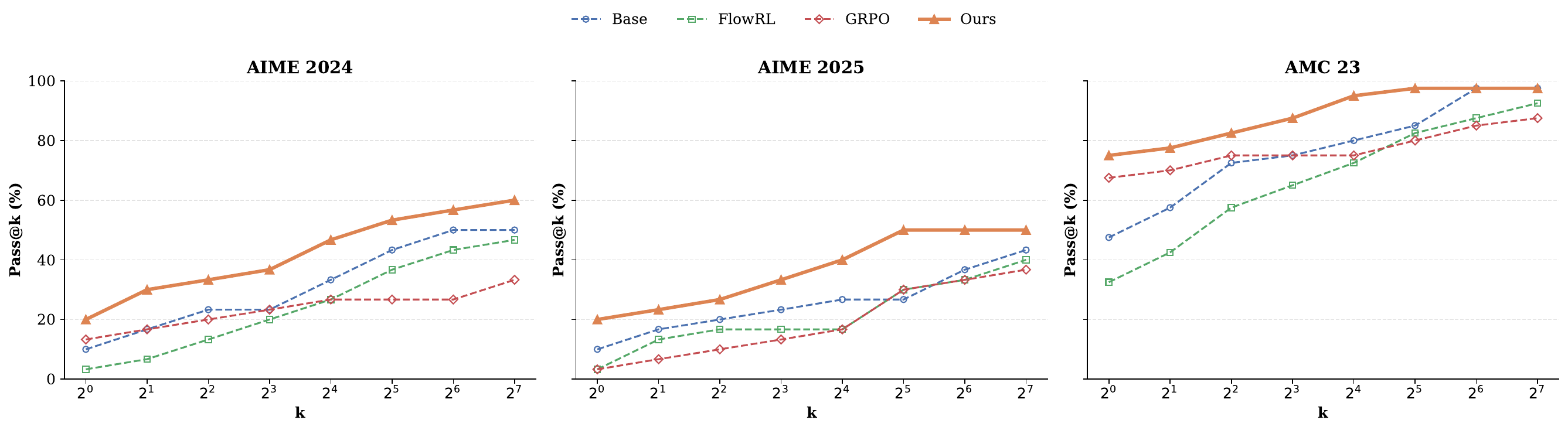}
    \caption{Pass@k comparison on AIME 2024, AIME 2025, and AMC 23 benchmarks using Qwen2.5-7B with FlowRL and GRPO and Ours.}
    \label{pass_at_k}
\end{figure*}

As shown in Figure~\ref{pass_at_k}, we further compare the Pass@K metric and observe that our method significantly surpasses the baseline base model across all evaluated settings.

\begin{table}[htbp]
\centering
\caption{Performance on out-of-distribution (OOD) general-domain benchmarks.}
\label{general_domain}
\small
\begin{tabular}{lcccc}
\toprule
\textbf{Models} & \textbf{MMLU-PRO} & \textbf{GPQA (Avg@4)} \\
\midrule
Qwen2.5-7B & 48.7 & 32.4 \\
+GRPO & 52.1 &  38.9 \\
\textbf{+ProGRPO (Ours)} & \textbf{54.3} & \textbf{42.3} \\
\bottomrule
\end{tabular}
\end{table}

In addition, we evaluate the generalization performance of our model under OOD settings. As shown in Table~\ref{general_domain}, our method still maintains a clear advantage over GRPO.

In summary, we have achieved consistent and significant performance improvements across varying model scales and architectures. These findings provide strong empirical support for our proposed confidence-based strategy, demonstrating its capability to substantially enhance both model diversity and generalization ability.

\subsection{Analysis of Training}

\begin{figure*}
    \centering
    \includegraphics[width=1\linewidth]{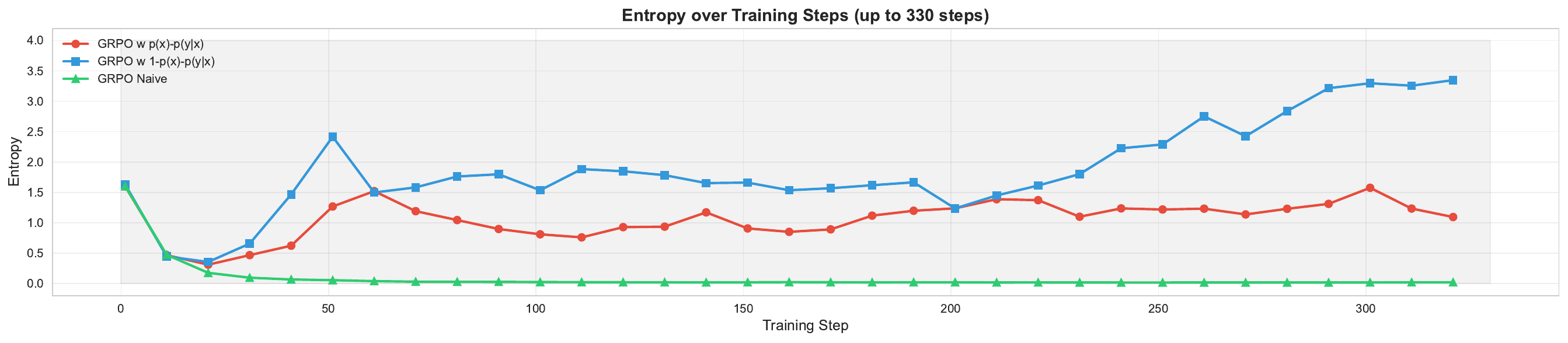}
    \caption{Training entropy across optimization steps for different methods. Higher entropy indicates increased exploration during policy optimization.}
    \label{training_entropy}
\end{figure*}

We continuously monitored the evolution of entropy during training. As shown in the Figure~\ref{training_entropy}, entropy first decreases, then increases, and eventually stabilizes. This behavior can be interpreted as follows: in the early stage of training, the model mainly focuses on learning a small number of correct answers, causing the predictive distribution to contract and entropy to decrease. As training progresses and the number of correct answers within a group increases, the model begins to allocate probabilities more evenly across samples, leading to a smoother output distribution and a corresponding rise in entropy, which eventually stabilizes. In contrast, GRPO consistently reinforces the most confident answers, driving probability mass toward a few samples and resulting in entropy collapse.

\begin{table}[htbp]
\centering
\caption{Dataset-level Evaluation on AIME 2024: Accuracy and Diversity of Correct Solutions}
\label{aime_eval_compare_with_GRPO}
\resizebox{\columnwidth}{!}{\begin{tabular}{lccccc}
\toprule
\textbf{Setting} & \textbf{Distinct-2} & \textbf{Self-BLEU} & \textbf{Semantic Cosine} \\
\midrule
GRPO & 0.1693 & 0.9299 & 0.9725 \\
\textbf{ProGRPO (Ours)} & \textbf{0.1443} & \textbf{0.6746} & \textbf{0.9233} \\
\bottomrule
\end{tabular}}
\end{table}

We also investigate how diverse are the generations after the model's training. Table~\ref{aime_eval_compare_with_GRPO} reports three diversity-related metrics — Distinct-2, Self-BLEU, and Semantic Cosine — computed at the dataset level for correct solutions on AIME 2024, where sentence representations are obtained using all-MiniLM-L6-v2 \citep{minilm}. Compared with the GRPO baseline, our method achieves substantially lower Self-BLEU and Semantic Cosine scores, indicating reduced lexical and semantic redundancy among generated solutions. Although Distinct-2 is slightly lower, this suggests that the improved diversity primarily stems from higher-level structural and semantic variation in reasoning rather than surface-level n-gram diversification. Overall, these results demonstrate that our approach encourages more diverse yet valid reasoning trajectories, consistent with the observed gains under higher-entropy decoding.

\begin{figure}[!ht]
    \centering
    \begin{subfigure}[b]{\linewidth}
        \centering
        \includegraphics[width=\linewidth]{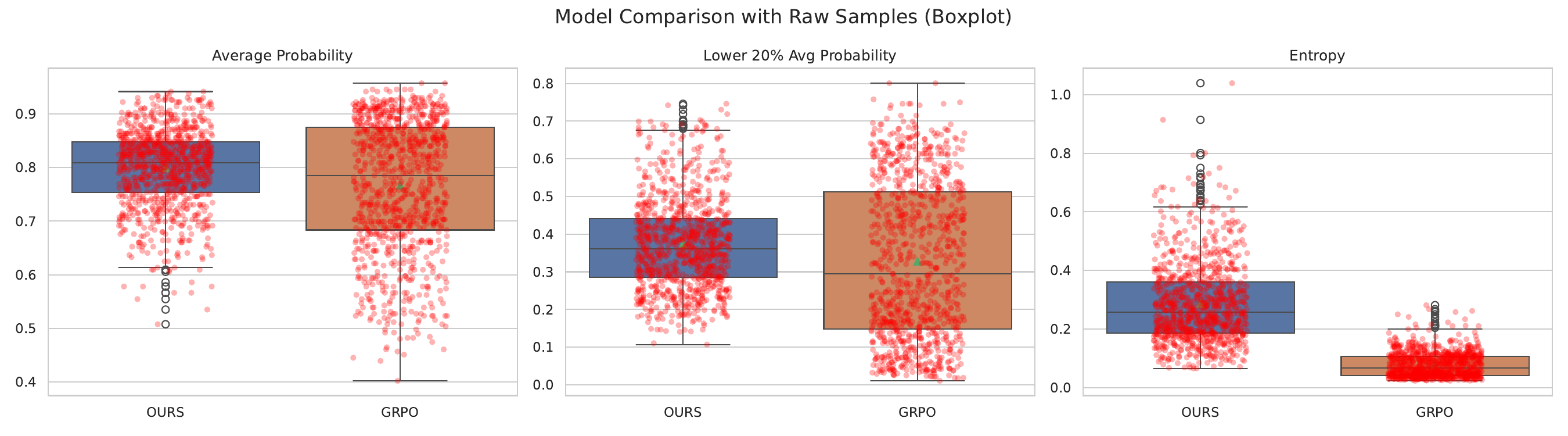}
        \caption{Boxplot comparison of model performance (OURS vs GRPO).}
        \label{boxplot_comparison}
    \end{subfigure}
    \hfill
    \begin{subfigure}[b]{\linewidth}
        \centering
        \includegraphics[width=\linewidth]{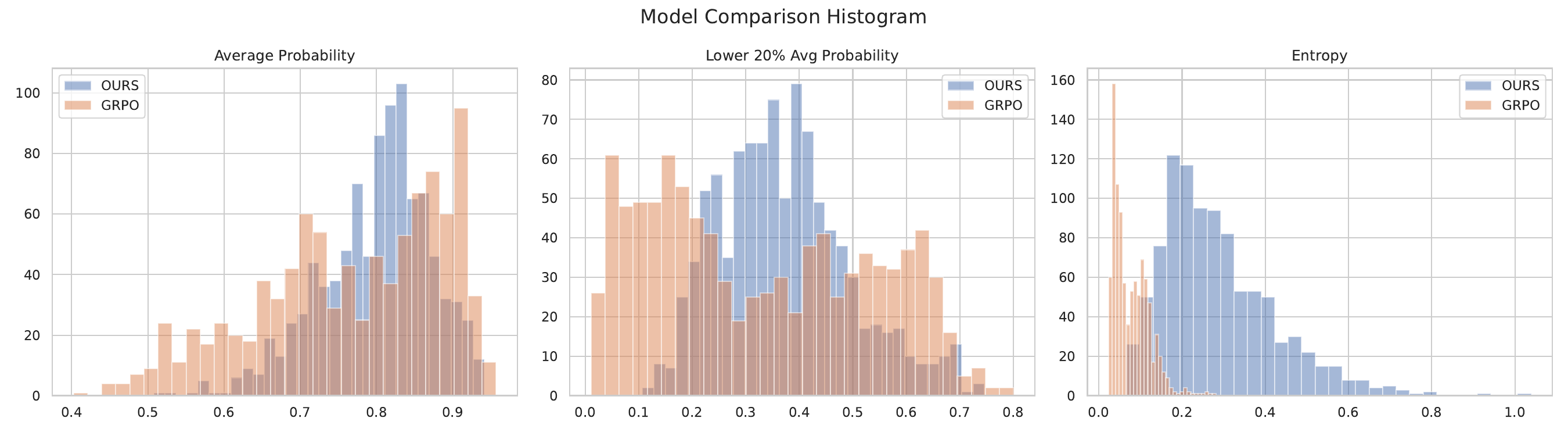}
        \caption{Histogram comparison of model performance (OURS vs GRPO).}
        \label{histogram_comparison}
    \end{subfigure}
    \caption{Comparison of model performance across three metrics (average probability, lower 20\% probability, and entropy), with statistics computed over 32 rollouts per sample using the AIME2024 dataset.}
    \label{model_comparison}
\end{figure}

Overall, ProGRPO method demonstrates superior performance to GRPO in both reliability and diversity: it achieves higher average probabilities, exhibits greater stability for low-probability tokens, and generates richer outputs. This conclusion aligns with the comparative results presented in Figure~\ref{boxplot_comparison} and Figure~\ref{histogram_comparison}.

\begin{figure}[!htb]
    \centering
    \begin{subfigure}[t]{0.48\linewidth}
        \centering
        \includegraphics[width=\linewidth]{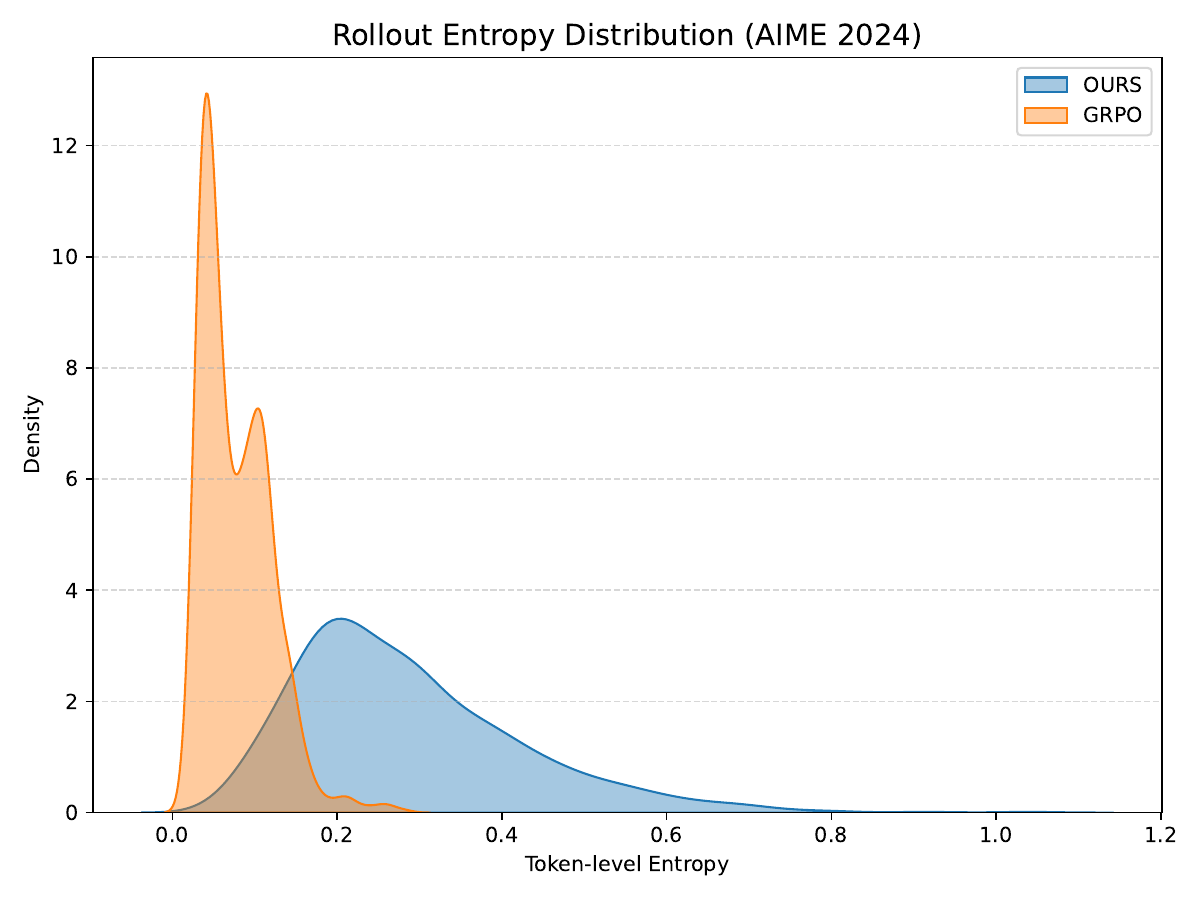}
        \caption{Kernel density estimation (KDE).}
        \label{entropy_kde}
    \end{subfigure}
    \hfill
    \begin{subfigure}[t]{0.48\linewidth}
        \centering
        \includegraphics[width=\linewidth]{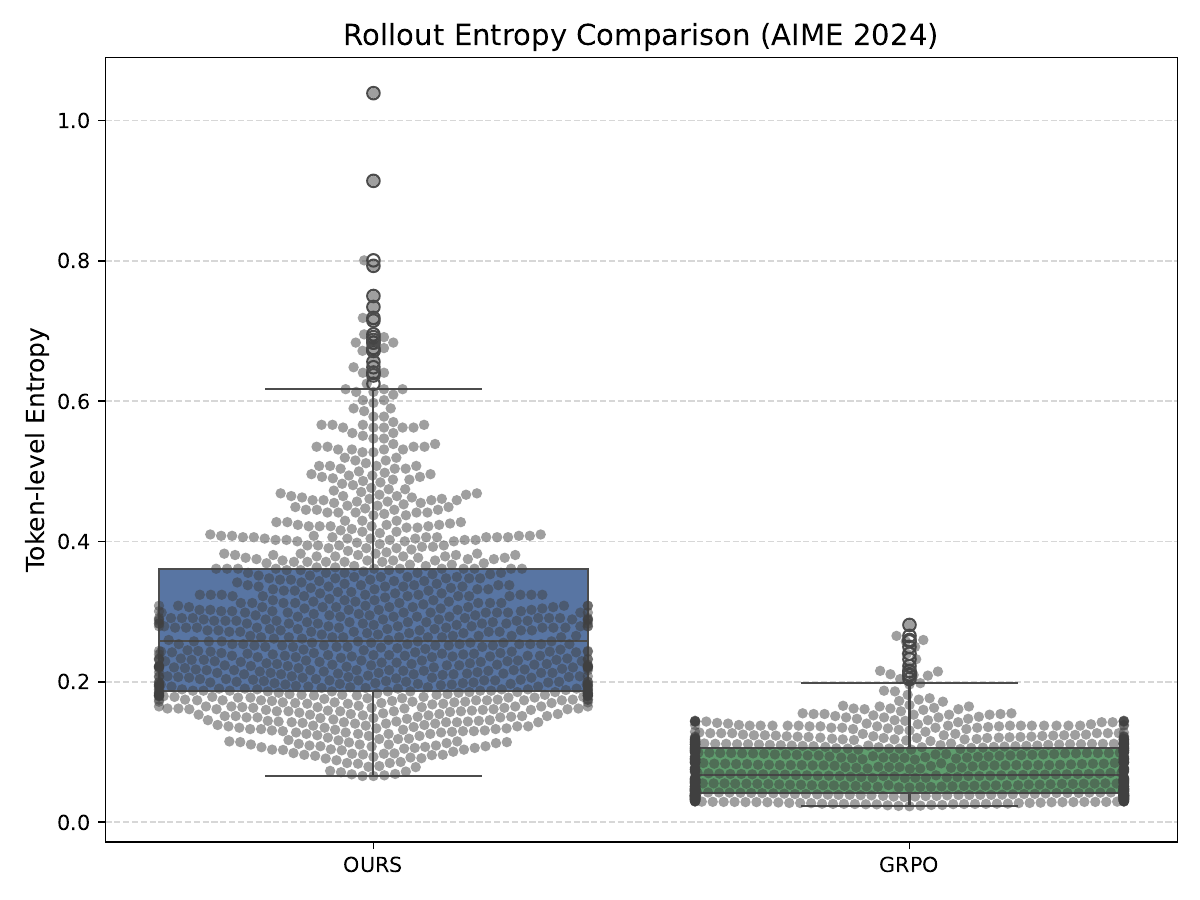}
        \caption{Boxplot of rollout entropy.}
        \label{entropy_boxplot}
    \end{subfigure}

    \caption{Comparison of rollout token-level entropy on AIME~2024 between OURS and the GRPO baseline.}
    \label{rollout_entropy_aime2024}
\end{figure}

As shown Figure~\ref{rollout_entropy_aime2024}, we conduct a systematic analysis of the output entropy on the AIME dataset. The results show that, compared with models trained using GRPO, our method significantly increases the entropy of the output distribution while maintaining comparable Pass@1 performance. Furthermore, when combined with the Pass@k metric, we observe that the model continues to achieve stable improvements in Pass@k under higher entropy levels. This indicates that the increased entropy does not arise from randomization, but rather from generating a more diverse set of valid reasoning paths while preserving solution correctness.

\subsection{Ablation}

\begin{figure}[!htb]
    \centering
    \includegraphics[width=1\linewidth]{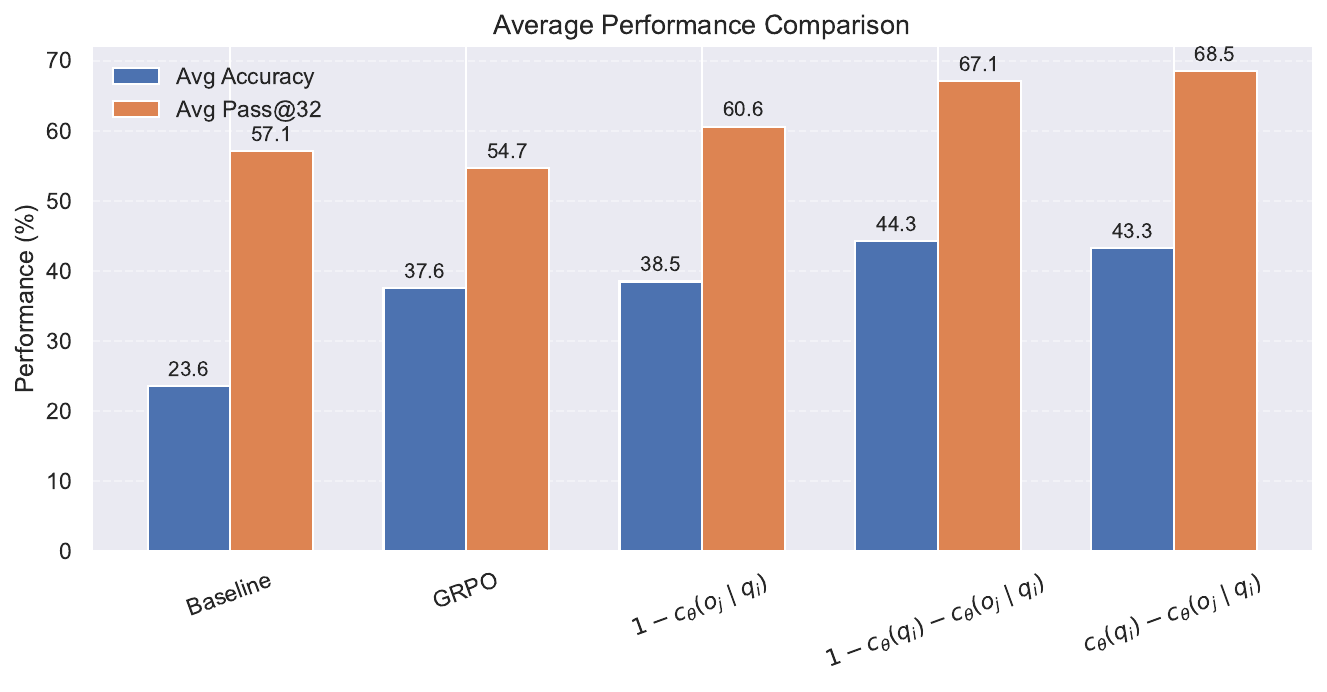}
    \caption{Ablation study of average pass@k performance under different advantage formulations.}
    \label{ablation_pic}
\end{figure}

As shown in Figure~\ref{ablation_pic}, compared with GRPO, our proposed algorithm consistently achieves stable and significant performance improvements across different models and advantage formulations. In particular, the effectiveness and robustness of our method are consistently validated when encouraging low-probability answers, as well as under advantage designs that combine perplexity — both low and high — with low-probability answer incentives.

\begin{table}[htbp]
\centering
\caption{Impact of the advantage reweighting coefficient $\alpha$ in Eq.~\ref{main_equa} on ProGRPO performance, averaged over six benchmarks.}
\label{ablation_alpha}
\begin{tabular}{lcccc}
\toprule
\textbf{Alpha} & \textbf{Pass@1} & \textbf{Pass@32} \\
\midrule
0 (Pure GRPO) & 37.6 & 54.7 \\
0.3 & \textbf{43.3} & \textbf{68.5} \\
0.7 & 40.9 & 59.8 \\
1.0 & 39.0 & 56.2 \\
\bottomrule
\end{tabular}
\end{table}

As shown in Table \ref{ablation_alpha}, the averaged results across six benchmarks demonstrate that the hyperparameter $\alpha$ in Equation X has a significant impact on model performance. When $\alpha=0$, the model reduces to pure GRPO and yields relatively weak performance. Increasing $\alpha$ to 0.3 introduces a moderate confidence-based advantage reweighting, leading to substantial improvements in both Pass@1 and Pass@32, which indicates more effective reinforcement of high-quality reasoning trajectories while maintaining training stability. However, further increasing $\alpha$ (e.g., to 0.7 or 1) degrades performance, suggesting that overly strong confidence signals can dominate the advantage function and weaken the supervision from the original reward. Overall, $\alpha=0.3$ achieves the best balance between performance and stability.

In summary, $c_\theta(q_i) - c_\theta(o_j \mid q_i)$ with $\alpha=0.3$ is a more robust choice, whereas $1 - c_\theta(q_i) - c_\theta(o_j \mid q_i))$ tends to be more exploratory, as it encourages larger updates on harder problems compared to easier ones.


\section{Related Work}

\textbf{Reasoning Models.} Recently, reinforcement learning–driven reasoning models have typically generated explicit and lengthy chains of thought before producing final answers, as exemplified by models such as OpenAI o1 \citep{openaio1}, DeepSeek \citep{grpo}, Kimi \citep{team2025kimi}, and Qwen \citep{bai2023qwen} \citep{qwen2.5} \citep{qwen3}. Within this paradigm, reinforcement learning with verifiable rewards has become a dominant post-training approach, with widely adopted algorithms including GRPO \citep{grpo}, GSPO \citep{gspo}, DAPO \citep{yu2025dapo}, and CISPO \citep{cispo}. However, the generalization ability of reasoning models remains a critical concern. Prior studies have shown that under the Pass@k metric, the performance advantage of RL-fine-tuned models over their base counterparts often diminishes—or even vanishes—as k increases \citep{doseRLpassk}. This phenomenon highlights a fundamental challenge arising from insufficient exploration mechanisms in reinforcement learning.

\textbf{Exploration in Reinforcement Learning.} Exploration and entropy collapse have long been central challenges in reinforcement learning. Prior work has explored the use of entropy-based signals \citep{entropy_perspective} to guide exploration, yet these approaches often yield limited empirical improvements. Other studies introduce entropy regularization \cite{cui2025entropy, wang2025reinforcement} to enhance exploratory behavior, but they still face challenges in terms of stability and effectiveness. Recently, FlowRL \citep{zhu2025flowrl} redefines the reward function to assign different scores to different reasoning paths (reward matching), thereby improving the model’s exploratory capabilities—while also effectively balancing the confidence across reasoning paths.

Inspired by research on model confidence \citep{li2025confidence} and FlowRL \citep{zhu2025flowrl}, we revisit the relationship between entropy collapse and confidence in policy learning. When the model is overly confident \citep{li2025confidence}, the policy tends to determinism, leading to entropy collapse. Motivated by this observation, we pose a complementary question: what happens when the model is insufficiently confident? To address this, we propose a confidence-balancing method based on the correctness of answers, which explicitly regulates model confidence and introduces a new research paradigm for achieving a better trade-off between exploration and stability.
\section{Conclusion}

In this paper, we propose \textbf{ProGRPO}, a novel algorithm approached from the perspective of generative probabilities, designed to address the severe entropy collapse phenomenon observed during RLVR training. By introducing Low-Probability Token Length Normalization and a confidence-aware Advantage Reweighting mechanism (ARM), ProGRPO effectively mitigates mode collapse while preserving the model's reasoning capabilities. 

Empirical results demonstrate that our method not only achieves competitive performance on standard metrics but also maintains a significant advantage in multi-sample settings (e.g., Pass@k), indicating a robust capability to generate diverse and correct solutions. Ultimately, ProGRPO presents a novel and effective solution to the fundamental Exploration-Exploitation trade-off in LLM reasoning tasks, paving the way for more stable and exploratory reinforcement learning paradigms.

\section*{Impact Statement}

This paper aims to advance the field of Machine Learning by improving the stability and diversity of reinforcement learning with verifiable rewards for large language models. Our proposed method focuses on mitigating mode collapse during policy optimization, thereby encouraging more diverse and robust reasoning behaviors without introducing new model capabilities or application domains.

The techniques presented in this work are primarily methodological and are intended to enhance existing training paradigms for reasoning and code generation tasks. We do not anticipate any significant negative societal or ethical consequences beyond those commonly associated with large language models, such as issues related to misuse or over-reliance, which are not exacerbated by our approach. Overall, we believe that this work contributes positively to the reliability and robustness of machine learning systems.

\bibliography{example_paper}
\bibliographystyle{icml2026}

\newpage
\appendix
\onecolumn

\section{Theoretical Justification} \label{Theoretical_Justification}

In this section, we provide a rigorous mathematical justification for the proposed ProGRPO framework. Under the RLVR (Reinforcement Learning with Verifiable Rewards) setting, we demonstrate how ProGRPO overcomes the fundamental limitations of standard GRPO through its Adaptive Margin Reward (AMR) mechanism.

\subsection{Preliminaries and the Homogeneity Limitation}

Let $q$ be a prompt sampled from a dataset $\mathcal{D}$, and $\pi_\theta$ be the policy model. In each iteration, GRPO samples a group of $G$ outputs $\mathcal{G} = \{o_1, \dots, o_G\}$.
\begin{itemize}
    \item \textbf{Binary Rewards:} $r(o) \in \{0, 1\}$.
    \item \textbf{Correctness Subsets:} $\mathcal{O}^+ = \{o \in \mathcal{G} \mid r(o)=1\}$ with cardinality $K = |\mathcal{O}^+|$.
    \item \textbf{Standard GRPO Statistics:} The mean reward $\mu$ and standard deviation $\sigma$ are:
    \begin{equation}
        \mu = \frac{K}{G}, \quad \sigma = \sqrt{\frac{K}{G}\left(1 - \frac{K}{G}\right)} + \epsilon
    \end{equation}
\end{itemize}

\begin{lemma}[Homogeneity of Advantage]
\label{lemma:homogeneity}
In standard GRPO, for any two distinct correct responses $o_i, o_j \in \mathcal{O}^+$, the advantage values are identical:
\begin{equation}
    A(o_i) = A(o_j) = \frac{1 - \mu}{\sigma} \triangleq A_{\text{pos}} > 0
\end{equation}
\end{lemma}

\begin{proof}
Since $r(o_i) = r(o_j) = 1$, the linear transformation $A = (r - \mu)/\sigma$ maps all elements of $\mathcal{O}^+$ to the same scalar $A_{\text{pos}}$. Consequently, the gradient update $\nabla_\theta \mathcal{J} \propto \sum A(o) \nabla_\theta \log \pi_\theta(o)$ treats all correct trajectories indiscriminately. If $\pi_\theta(o_i) > \pi_\theta(o_j)$ initially, the model enters a positive feedback loop, exponentially increasing $\pi_\theta(o_i)$ while suppressing other valid paths $o_j$. This leads to \textbf{Entropy Collapse}.
\end{proof}

\subsection{Theorem 1: Convergence to Confidence Equilibrium and Difficulty Calibration}

The ProGRPO advantage for $o_i \in \mathcal{O}^+$ is defined as:
\begin{equation}
    \tilde{A}(o_i) = A_{\text{pos}} + \alpha \left( \bar{c}_\mathcal{G} - c_\theta(o_i \mid q) \right)
\end{equation}
where $\bar{c}_\mathcal{G} = \frac{1}{G} \sum_{j=1}^G c_\theta(o_j \mid q)$ is the prompt-specific group baseline.

\begin{theorem}
\label{thm:equilibrium}
The AMR mechanism stabilizes policy dynamics within $\mathcal{O}^+$ by: (i) inducing a Maximum Entropy state through negative feedback, and (ii) eliminating the difficulty bias across different prompts.
\end{theorem}

\begin{proof}
\textbf{Part 1: Negative Feedback for Diversity.} 
Consider $o_1, o_2 \in \mathcal{O}^+$. If $c_\theta(o_1 \mid q) > c_\theta(o_2 \mid q)$, then $\tilde{A}(o_1) < \tilde{A}(o_2)$. This differential advantage ensures that over-optimized paths receive a smaller reinforcement signal than under-explored ones. Equilibrium is reached only when $\tilde{A}(o_1) = \tilde{A}(o_2)$, implying $c_\theta(o_1 \mid q) = c_\theta(o_2 \mid q)$, which corresponds to a uniform distribution over the success manifold.

\textbf{Part 2: Why Prompt-Specific $\bar{c}_\mathcal{G}$ is Essential.} 
Let $D(q)$ represent the intrinsic difficulty of prompt $q$. For easy prompts, the model's absolute confidence $c_\theta$ is naturally high, while for hard prompts, $c_\theta$ is low. 
\begin{itemize}
    \item \textit{Failure of Global Baseline:} If we used a fixed global threshold $\tau$ instead of $\bar{c}_\mathcal{G}$, the term $(\tau - c_\theta)$ would be consistently negative for all easy prompts (penalizing correct answers) and positive for all hard prompts (ignoring diversity).
    \item \textit{Calibration via $\bar{c}_\mathcal{G}$:} By defining the advantage relative to the group mean $\bar{c}_\mathcal{G}$, we isolate the \textbf{intra-prompt path discrepancy} from the \textbf{inter-prompt difficulty noise}. 
\end{itemize}
Mathematically, let $c_\theta(o \mid q) = f(q) + \delta(o)$, where $f(q)$ is the prompt difficulty component and $\delta(o)$ is the path-specific variation. Then:
\begin{equation}
    \bar{c}_\mathcal{G} - c_\theta(o_i \mid q) = \left( f(q) + \frac{1}{G}\sum \delta(o_j) \right) - \left( f(q) + \delta(o_i) \right) = \bar{\delta} - \delta(o_i)
\end{equation}
The prompt-specific bias $f(q)$ is canceled out. This ensures that the diversity pressure is applied consistently across the entire dataset, regardless of whether a prompt is easy or difficult.
\end{proof}

\subsection{Theorem 2: Semantic Diversity vs. Syntactic Fluency}

\begin{theorem}
AMR induces semantic-level diversity on reasoning paths while preserving the syntactic certainty of functional segments.
\end{theorem}

\begin{proof}
Let an output $o$ be decomposed into functional tokens $S_{\text{func}}$ (e.g., ``The answer is'') and reasoning tokens $S_{\text{reason}}$. The confidence $c_\theta$ is computed over the low-probability set $\mathcal{T}^{\text{low}}$:
\begin{equation}
    \mathcal{T}^{\text{low}} = \{ t \in [1, |o|] \mid p_\theta(o_t \mid o_{<t}) < \text{bottom-}20\% \text{ threshold} \}
\end{equation}
For functional tokens, $p_\theta(t) \approx 1$ due to grammatical determinism, hence $S_{\text{func}} \cap \mathcal{T}^{\text{low}} = \emptyset$. Consequently:
\begin{equation}
    \frac{\partial \tilde{A}}{\partial \pi(t)} = 0, \quad \forall t \in S_{\text{func}}
\end{equation}
Unlike standard Entropy Maximization which penalizes all tokens, AMR targets only the branching points in $S_{\text{reason}}$, protecting the model's linguistic fluency.
\end{proof}

\subsection{Theorem 3: Preservation of Correctness}

\begin{theorem}
The AMR mechanism strictly bounds exploration within the valid reward landscape and does not encourage incorrect responses $o \in \mathcal{O}^-$.
\end{theorem}

\begin{proof}
For an incorrect response $o_{\text{neg}}$ where $r(o_{\text{neg}})=0$, the base advantage is $A_{\text{neg}} = -\mu/\sigma < 0$. Under AMR:
\begin{equation}
    \tilde{A}_{\text{neg}} = A_{\text{neg}} + \alpha \left( \bar{c}_\mathcal{G} - c_\theta(o_{\text{neg}} \mid q) \right)
\end{equation}
To ensure $o_{\text{neg}}$ remains suppressed, we require $\tilde{A}_{\text{neg}} < 0$. This is satisfied when:
\begin{equation}
    \alpha < \frac{|A_{\text{neg}}|}{\sup | \bar{c}_\mathcal{G} - c_\theta |}
\end{equation}
Since $|A_{\text{neg}}|$ is significantly negative in the early-to-mid training stages and $\alpha$ is a small hyperparameter (e.g., 0.1), the sign of the gradient remains negative. Thus, AMR acts as a modulation of the penalty magnitude rather than a reversal of the objective, ensuring incorrect paths are never reinforced.
\end{proof}

\subsection{Theorem 4: Implicit Entropy Regularization on the Success Manifold}

\begin{theorem}
\label{thm:entropy}
Under the Reinforcement Learning with Verifiable Rewards (RLVR) setting, 
the Advantage Modulation Rule (AMR) in ProGRPO induces an \emph{implicit entropy-maximizing bias} over the set of correct solutions $\mathcal{O}^+$.
Specifically, while preserving correctness, AMR promotes a high-entropy policy over $\mathcal{O}^+$, thereby mitigating mode collapse.
\end{theorem}

\begin{proof}
\textbf{1. Entropy collapse in standard GRPO.}
In standard GRPO, Lemma~\ref{lemma:homogeneity} shows that all correct solutions 
$o \in \mathcal{O}^+$ share the same positive advantage $A_{\text{pos}}$.
The policy gradient restricted to $\mathcal{O}^+$ is therefore
\begin{equation}
\nabla_\theta \mathcal{J}_{\text{GRPO}}\big|_{\mathcal{O}^+}
\propto
\sum_{o_i \in \mathcal{O}^+} 
A_{\text{pos}} \nabla_\theta \log \pi_\theta(o_i \mid q).
\end{equation}
This corresponds to a uniform likelihood amplification over competing correct trajectories.
From an information-theoretic perspective, such undifferentiated reinforcement causes probability mass to concentrate on the initially slightly more likely paths (e.g., shorter or syntactically simpler solutions), leading to entropy collapse:
\[
H(\pi_\theta \mid \mathcal{O}^+) \to 0.
\]

\textbf{2. AMR as an entropy-promoting mechanism.}
Consider the Shannon entropy of the policy restricted to $\mathcal{O}^+$:
\begin{equation}
H_{\mathcal{O}^+}
= -\sum_{i=1}^{K} p_i \log p_i,
\quad
p_i
=
\frac{\pi_\theta(o_i \mid q)}
{\sum_{j \in \mathcal{O}^+} \pi_\theta(o_j \mid q)}.
\end{equation}
The gradient of $H_{\mathcal{O}^+}$ encourages the distribution to become uniform over $\mathcal{O}^+$.

In ProGRPO, the modulated advantage for correct solutions is
\begin{equation}
\tilde{A}(o_i)
=
A_{\text{pos}}
+
\alpha \left(
\mathbb{E}_{j \sim \mathcal{G}}[c_\theta(o_j \mid q)]
-
c_\theta(o_i \mid q)
\right),
\end{equation}
where the confidence score $c_\theta(o_i \mid q)$ serves as a monotonic proxy for
$\log \pi_\theta(o_i \mid q)$ (e.g., defined as an average of token-level probabilities).

Ignoring the constant term $A_{\text{pos}}$, the AMR-induced gradient direction is
\begin{equation}
\nabla_\theta \mathcal{J}_{\text{AMR}}
\propto
\sum_{o_i \in \mathcal{O}^+}
\alpha (\bar{c} - c_i)
\nabla_\theta \log \pi_\theta(o_i \mid q),
\end{equation}
where $\bar{c}$ denotes the group-average confidence.
Paths with higher-than-average confidence are down-weighted, while lower-confidence paths are amplified.
This update direction is aligned with that of minimizing the KL divergence between the policy restricted to $\mathcal{O}^+$ and the uniform distribution, thereby implicitly encouraging higher entropy.

\textbf{3. Stationary points.}
At a stationary point, $\nabla_\theta \mathcal{J} = 0$, which within $\mathcal{O}^+$ requires
\[
\tilde{A}(o_i) = \tilde{A}(o_j), \quad \forall i,j.
\]
By the definition of AMR, this implies
\[
c_\theta(o_i \mid q) = c_\theta(o_j \mid q), \quad \forall i,j,
\]
indicating equalized confidence (and thus probability mass) across all correct solutions.
Consequently, the induced policy attains a maximum-entropy configuration over $\mathcal{O}^+$.

Therefore, ProGRPO preserves correctness while implicitly steering the policy toward a high-entropy distribution on the success manifold, effectively preventing mode collapse.
\end{proof}

\section{Additional Results}

\begin{figure}[H]
    \centering
    \includegraphics[width=0.7\linewidth]{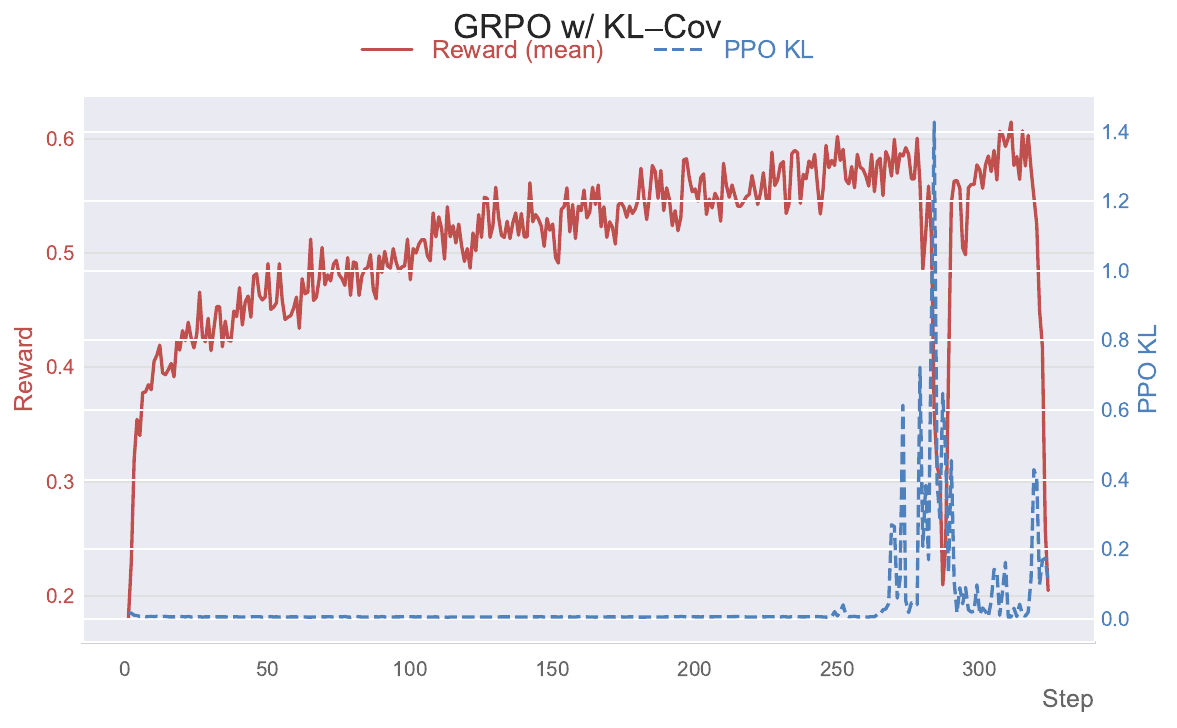}
    \caption{Reproduction of the best-performing method proposed in Entropy Mechanism \citep{cui2025entropy}}
    \label{entropy-based-rl-kl-cov}
\end{figure}

We reproduced the method proposed by \citet{cui2025entropy}. During our experiments, we observed that although the method can achieve the reported performance in some cases, the training process is extremely unstable and prone to collapse, as show in Figure~\ref{entropy-based-rl-kl-cov}. This highlights the practical challenges of reproducing the approach and motivates the need for more stable alternatives.

\begin{figure}[!htp]
    \centering
    \begin{subfigure}[t]{0.48\linewidth}
        \centering
        \includegraphics[width=\linewidth]{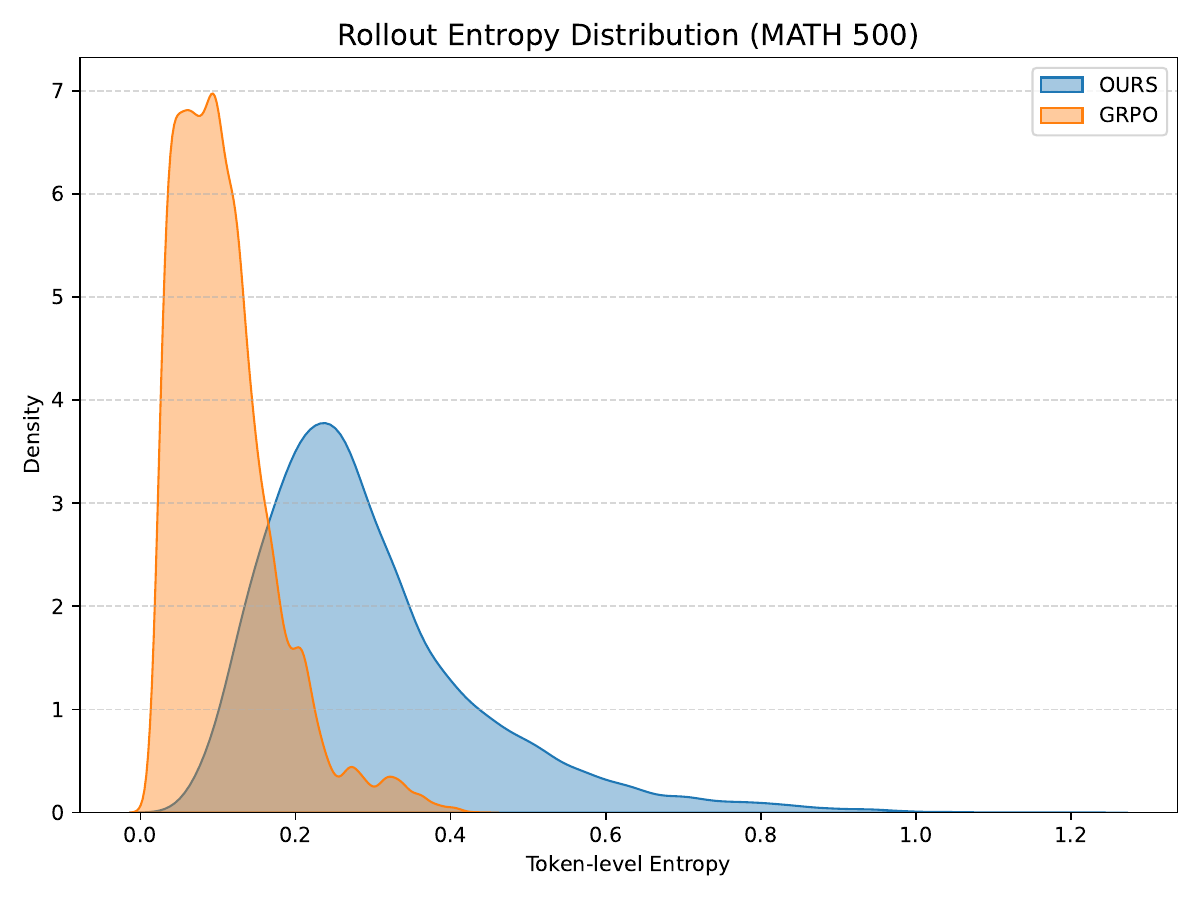}
        \caption{Kernel density estimation (KDE).}
        \label{entropy_kde_math500}
    \end{subfigure}
    \hfill
    \begin{subfigure}[t]{0.48\linewidth}
        \centering
        \includegraphics[width=\linewidth]{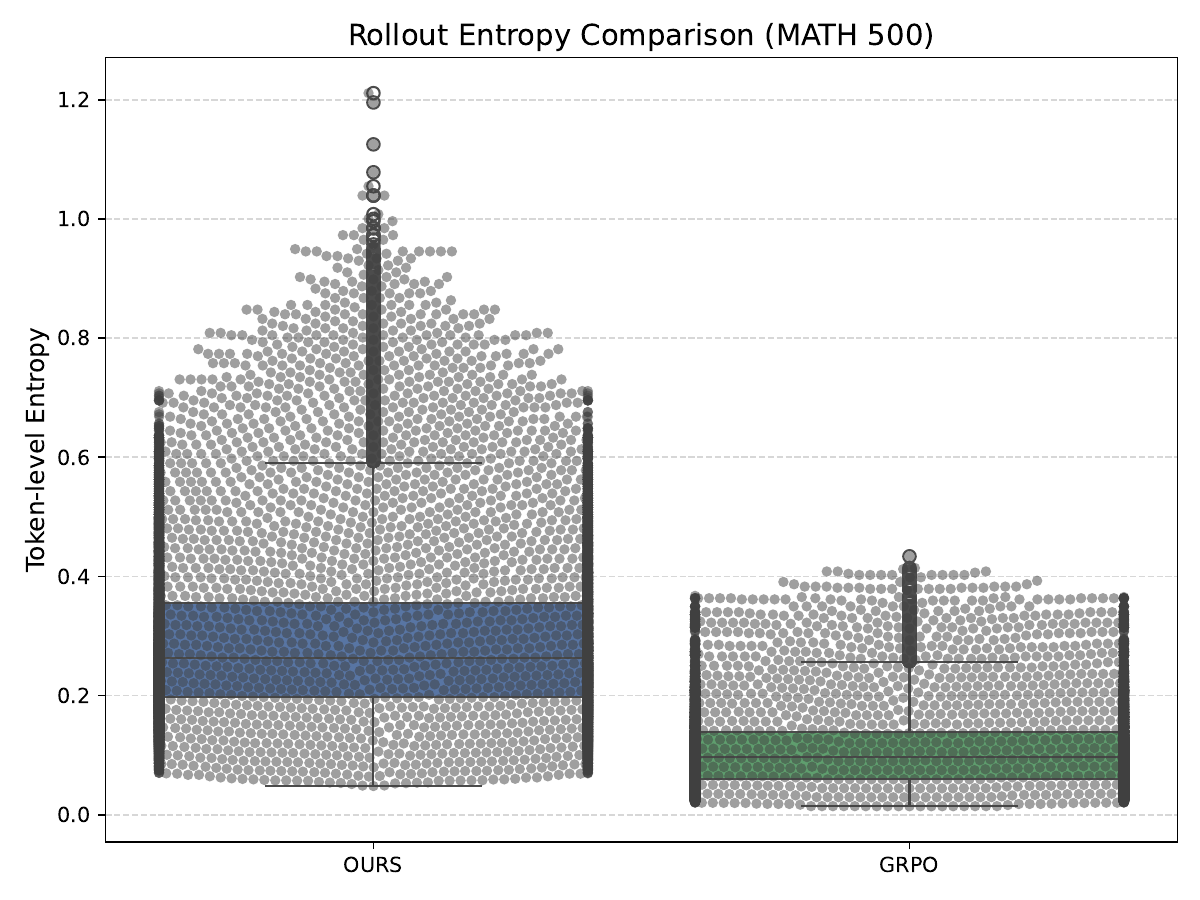}
        \caption{Boxplot of rollout entropy.}
        \label{entropy_boxplot_math500}
    \end{subfigure}
    \caption{Rollout token-level entropy comparison on Math~500 between OURS and the GRPO baseline.}
    \label{rollout_entropy_math500}
\end{figure}

Since the AIME2024 dataset contains only 30 samples, we also computed the entropy on the Math500 dataset (see Figure~\ref{rollout_entropy_math500}), and the results are consistent with the analysis on AIME2024.

\begin{figure}[!ht]
    \centering
    \includegraphics[width=1\linewidth]{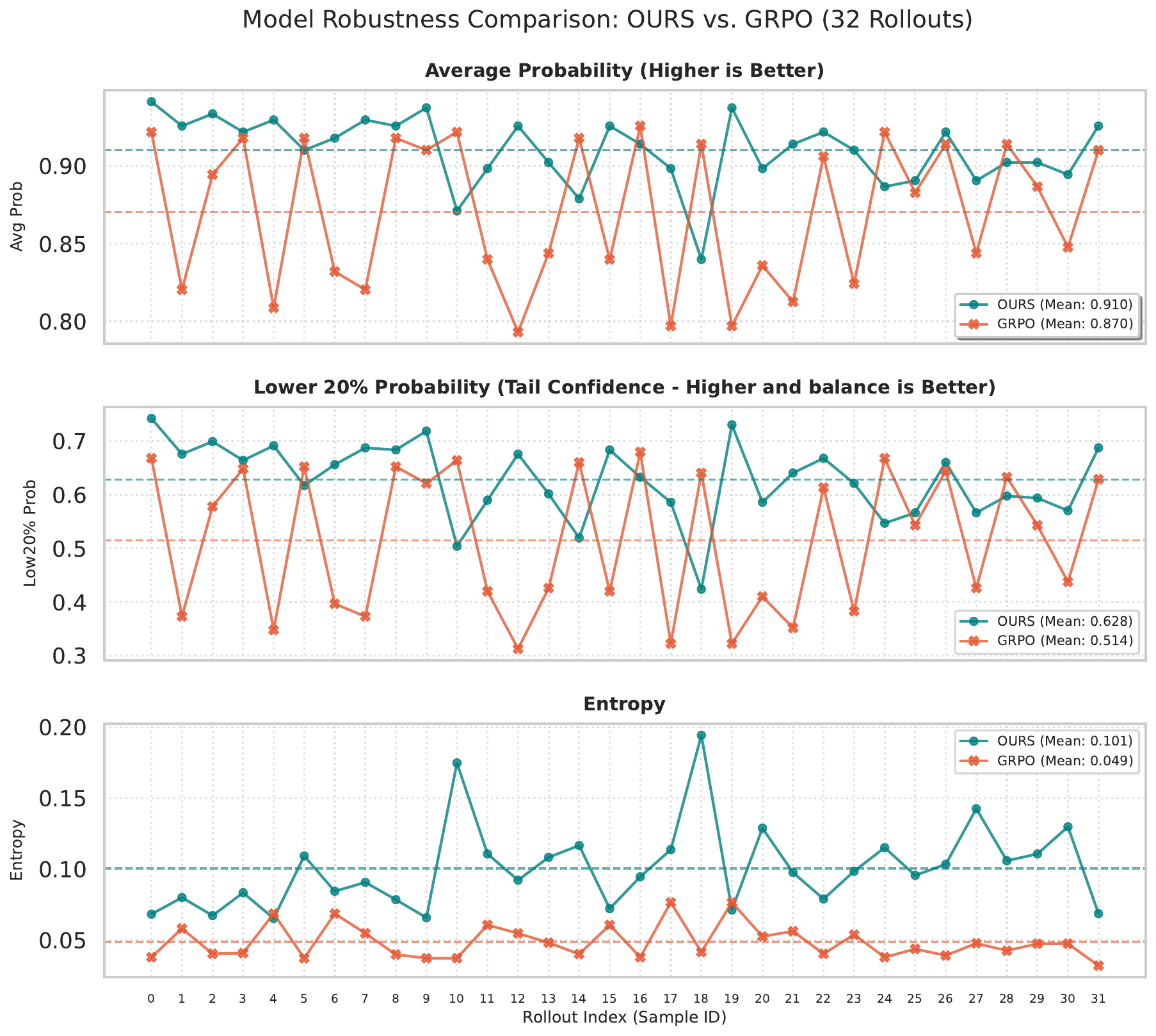}
    \caption{Per-sample analysis of 32 rollouts, showing average token probability and the mean of the lowest 20\% token probabilities.}
    \label{rollout_32}
\end{figure}

Analysis of a single sample from the AIME2024 dataset (Figure~\ref{rollout_32}) shows that our model achieves higher token-level entropy and more balanced probabilities across 32 rollouts. This is consistent with the inter-sample balancing mechanism in Equation~\ref{main_equa}, demonstrating that our model outperforms the GRPO baseline in both reliability and diversity of generation.

\begin{table*}[htp]
\centering
\small
\resizebox{\textwidth}{!}{\begin{tabular}{lccccccc}
\toprule
\textbf{Method} &
\textbf{AIME 2024} &
\textbf{AIME 2025} &
\textbf{AMC 23} &
\textbf{MATH 500} &
\textbf{Minerva} &
\textbf{Olympiad} &
\textbf{Avg} \\
\midrule

Baseline 
& 5.4 / 30.0 
& 2.5 / 20.0 
& 32.4 / 82.5 
& 54.7 / 92.6 
& 22.0 / \textbf{54.4}
& 24.6 / 63.1 
& 23.6 / 57.1 \\
        
GRPO
& 9.2 / 26.7
& 6.1 / 30.0
& 65.5 / 80.0
& 75.3 / 87.0
& 33.8 / 51.1
& 35.6 / 53.6
& 37.6 / 54.7 \\

with $1 - c_\theta(o_j \mid q_i)$
& 12.5 / 33.3
& 9.7 / 30.0
& 66.3 / \textbf{95.0}
& 73.7 / 91.2
& 32.7 / 52.9
& 36.0 / 61.1
& 38.5 / 60.6 \\

with $1 - c_\theta(q_i) - c_\theta(o_j \mid q_i)$
& 16.6 / 43.3
& 15.2 / 46.7
& \textbf{70.0} / 90.0
& \textbf{81.8} / \textbf{95.0}
& \textbf{36.4} / \textbf{57.4}
& \textbf{45.8} / \textbf{70.0}
& \textbf{44.3} / 67.1 \\

with $c_\theta(q_i) - c_\theta(o_j \mid q_i)$
& \textbf{21.3} / \textbf{53.3}
& \textbf{15.9} / \textbf{50.0}
& 67.2 / 92.5
& 80.5 / 94.2
& 32.0 / 53.3
& 42.7 / 67.5
& 43.3 / \textbf{68.5} \\

\bottomrule
\end{tabular}}
\caption{Ablation study on reward formulation. Each cell reports Acc / Pass@32 (\%).}
\label{ablation_reward_grpo}
\end{table*}

As shown in Table~\ref{ablation_reward_grpo}, these results highlight that relative confidence reweighting, rather than absolute confidence penalties, is essential for selectively attenuating dominant paths while promoting under-explored correct reasoning trajectories.

\section{}




\end{document}